\documentclass[journal]{IEEEtran}
\usepackage{color}
\usepackage{algorithm}
\usepackage{algorithmic}
\usepackage{graphicx}
\usepackage{subfigure} 
\usepackage{amsmath}
\usepackage{amsthm}
\usepackage{amssymb}
\usepackage{enumitem}
\usepackage{multirow}
\usepackage{url}

\usepackage{array}
\newcolumntype{L}[1]{>{\raggedright\let\newline\\\arraybackslash\hspace{0pt}}m{#1}}
\newcolumntype{C}[1]{>{\centering\let\newline  \\\arraybackslash\hspace{0pt}}m{#1}}
\newcolumntype{R}[1]{>{\raggedleft\let\newline \\\arraybackslash\hspace{0pt}}m{#1}}

\newtheorem{theorem}{Theorem}[section]
\newtheorem{lemma}[theorem]{Lemma}
\newtheorem{proposition}[theorem]{Proposition}

\newcommand{\NM}[2]{\| #1 \|_{#2}} 
\newcommand{\R}{\mathbb{R}}
\newcommand{\DifReg}{\ell_{1\text{-}2}}
\newcommand{\D}{\mathcal{D}}
\newcommand{\Px}[2]{\text{prox}_{#1}( #2 ) }
\newcommand{\SO}[2]{S_{#1}\left(#2\right)}
\newcommand{\sign}[1]{ \text{\,sign} ( #1 )}
\newcommand{\Diag}[1]{ \text{\,diag} \left(  #1 \right) }
\newcommand{\TV}[1]{ \text{TV}_{1\text{-}2} ( #1 )}

\newcommand{\Tr}[1]{\text{Tr}\left( #1 \right) }

\ifCLASSINFOpdf
\else
\fi
\begin{document}
%
\title{Fast Learning of Nonconvex $\ell_{1\text{-}2}$-Regularizer
 using the Proximal Gradient Algorithm}
%
%
%

\author{Quanming Yao,~\IEEEmembership{Member~IEEE,}
        James T. Kwok,~\IEEEmembership{Fellow~IEEE}
        and Xiawei Guo,~\IEEEmembership{Member~IEEE}
\thanks{The authors are with the Department of Computer Science and Engineering, 
	Hong Kong University of Science and Technology, Clear Water Bay,
	Hong Kong. E-mail: \{qyaoaa, jamesk, xguoae\}@cse.ust.hk.}
}

%
%

\markboth{IEEE TRANSACTIONS ON NEURAL NETWORKS AND LEARNING SYSTEMS}%
{}
%



\maketitle

\begin{abstract}
Recently, nonconvex regularizers have been shown to outperform traditional convex regularizers.
In particular,
the $\ell_{1\text{-}2}$ regularizer
(based on the difference of $\ell_1$- and $\ell_2$-norms)
yields better recovery performance than the $\ell_0$ and $\ell_1$-regularizers on
various tasks.
However, 
it is still challenging
to efficiently solve the resultant $\ell_{1\text{-}2}$ regularization problem.
As both the $\ell_1$- and $\ell_2$-norms are
not differentiable,
popular optimization algorithms cannot be used.
In this paper, 
we derive a cheap closed-form solution for the proximal step associated with the
$\ell_{1\text{-}2}$-regularizer.
This enables state-of-the-art proximal gradient algorithms to be used
for fast optimization.
We further extend the proposed solution to
low-rank matrix learning and the total variation model.
Experiments 
on both synthetic and real-world data sets
show that using the proximal gradient algorithm with 
the proposed solution
is much more efficient
than the state-of-the-art.
\end{abstract}


%
\IEEEpeerreviewmaketitle

\section{Introduction}
\label{sec:intro}


\IEEEPARstart{M}{achine}
learning problems are usually formulated as
\begin{align}
\min_{x} F(x) \equiv f(x) + \lambda g(x),
\label{eq:pro}
\end{align}
where $x \in \R^d$ is the model parameter, 
$f$ is a smooth loss function,
and $g$ is the regularizer.
The proper choice of regularization is usually the key to achieving good generalization performance for learning problems.

Recently,
sparse and low-rank regularization have attracted lots of attention.
Since direct minimization of the $\ell_0$-norm and rank function are NP-hard \cite{donoho2006sparse,candes2009exact},
convex regularizers have been commonly used in \eqref{eq:pro} instead.
Examples include the $\ell_1$-regularizer for compressed sensing \cite{hale2008fixed,beck2009fast}
and the nuclear norm regularizer for low-rank matrix or tensor completion \cite{candes2009exact,cai2010singular}.
Comparing with direct minimization with the $\ell_0$-norm or rank function,
such convex optimization problems are computationally tractable
where many convex optimization techniques can be applied.
However, such approximation may yield suboptimal performance due to the biased approximation to $\ell_0$-norm in the sense that $\ell_1$-norm is dominated by entries with large magnitudes, 
unlike $\ell_0$-norm in which all nonzero entries have equal contributions.

This motivates the design of many nonconvex regularizers,
such as the $\ell_p$-regularizer with $p \in (0, 1)$ \cite{chartrand2008iteratively},
capped $\ell_1$-norm \cite{zhang2010analysis}
and the
log-sum-penalty (LSP) \cite{candes2008enhancing}.
They usually produce better empirical performance than the convex $\ell_1$-norm.
Among them, 
a recently proposed nonconvex regularizer is the $\DifReg$-regularizer
\cite{esser2013method,yin2015minimization,lou2015computing},
which is written as a 
difference of the $\ell_1$- and $\ell_2$-norms of the parameter $x$:
\begin{align}
\NM{x}{1\text{-}2}
=\NM{x}{1} - \NM{x}{2}.
\label{eq:difl12}
\end{align}
A two dimensional case is plotted in Figure~\ref{fig:curve},
we can see that $\DifReg$-regularizer approaches the $x_1$-axis and $x_2$-axis closer as the values get smaller.
Hence,
it gives better approximation compared with above mentioned nonconvex regularizers.
Such intuition is later supported by many empirical evidences.
For example,
in compressed sensing problems with an ill-conditioned dictionary, 
the $\DifReg$-regularizer 
has been shown  to
outperform
existing convex and nonconvex regularizers on the reconstruction of sparse vectors
\cite{yin2015minimization,lou2015computing}.
In stochastic collocation,
it shows more reliable recovery on both sparse and non-sparse signals 
\cite{yan2017sparse}.
In image processing,
the total variation model  with the $\DifReg$-regularizer 
achieves state-of-the-art performance
\cite{lou2015weighted};
and the $\DifReg$-regularizer has also been extended to low-rank matrix learning problems,
achieving better performance than both the nuclear norm regularization and factorization approaches \cite{p2015phase,matruncated2016}.
Along with such good empirical performance,
its theoretical recovery guarantee has been recently established 
for compressed sensing \cite{yin2015minimization} and low-rank matrix completion \cite{matruncated2016},
which shows reliable recovery can be obtained under restricted isometry property \cite{donoho2006sparse}.

Despite its good empirical performance and sound theoretical guarantee,
optimization with the $\DifReg$-regularizer is much harder.
When the convex $\ell_1$-regularizer is used in \eqref{eq:pro},
many efficient algorithms have been developed,
examples including the alternating direction method of multipliers (ADMM) \cite{boyd2011distributed}, 
the Frank-Wolfe algorithm \cite{jaggi2013revisiting}
and the Bregman iterative algorithm \cite{yin2008bregman}.
However,
since the $\DifReg$-regularizer is nonconvex and has two nonsmooth components,
those popular optimization algorithms all fail.
As a result,
we need to seek help from more general optimization algorithms.
As \eqref{eq:difl12} represents this regularizer 
as the difference of two convex functions,
difference-of-convex programming (DCA) \cite{yuille2002concave} can naturally be used.
It is indeed the only given algorithm previously using the $\DifReg$-regularizer in above applications
\cite{yin2015minimization,lou2015computing,yan2017sparse,lou2015weighted,matruncated2016}.
However, each DCA iteration requires exact solving a convex subproblem which usually does not have a cheap closed-form solution,
and other numerical iteratively algorithms have to be used for solving the subproblem.
This makes DCA inefficient and slow in general \cite{gongZLHY2013,li2015accelerated}.
Sequential convex programming (SCP) \cite{zhaosong2012},
a DCA variant with lower iteration complexity,
can also be applied.
However, it still suffers from slow convergence \cite{li2015accelerated}.

\begin{figure*}[ht]
\centering
\subfigure[$\ell_1$-norm.]
{\includegraphics[width = 0.27\textwidth]{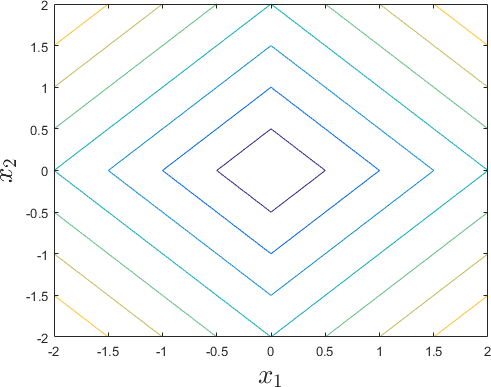}}
\quad\quad
\subfigure[Capped $\ell_1$-norm \cite{zhang2010analysis}.]
{\includegraphics[width = 0.27\textwidth]{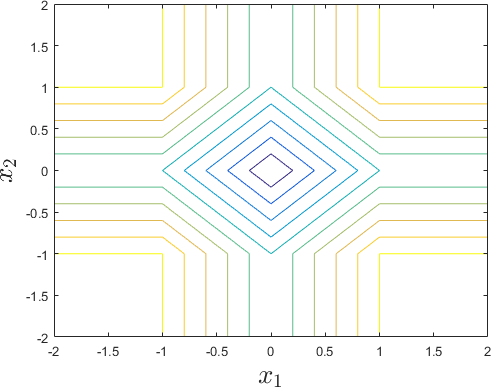}}
\quad\quad
\subfigure[$\ell_p$-norm ($p = 0.5$) \cite{chartrand2008iteratively}.]
{\includegraphics[width = 0.27\textwidth]{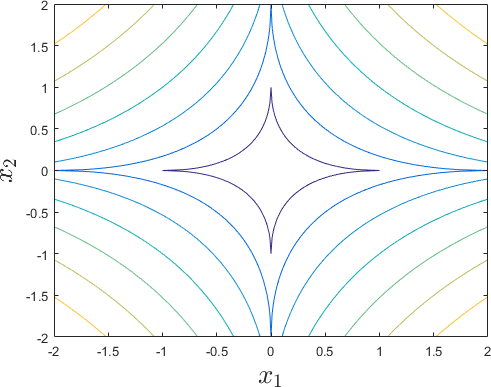}}

\subfigure[Log-sum-penalty \cite{candes2008enhancing}.]
{\includegraphics[width = 0.27\textwidth]{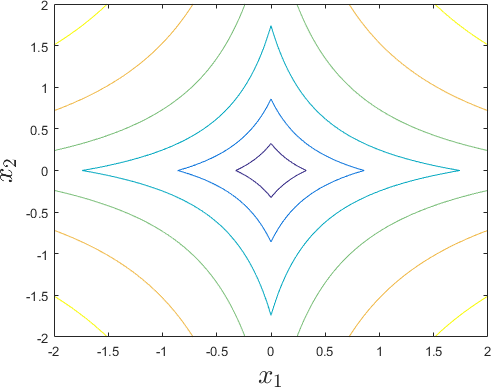}}
\quad\quad
\subfigure[$\ell_{1\text{-}2}$-regularizer \cite{yin2015minimization}.]
{\includegraphics[width = 0.27\textwidth]{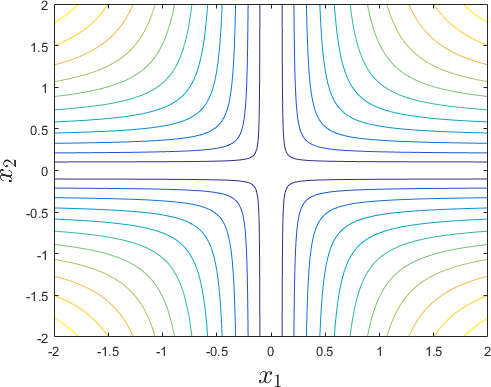}}

\caption{Level curves of the various regularizers.}
\label{fig:curve}
\vspace{-10px}
\end{figure*}

Recently,
the proximal gradient (PG) algorithm \cite{parikh2013proximal} has 
become a promising approach for \eqref{eq:pro},
which still requires $f$ to be smooth but allows the regularizer $g$ to be nonconvex and nonsmooth
\cite{gongZLHY2013,li2015accelerated,bolte2014proximal}.
It has been successfully used with some nonconvex regularizers,
such as the log-sum-penalty,
on the learning of sparse vectors \cite{gongZLHY2013} and low-rank matrices \cite{qyao2015icdm,lu2016nonconvex}.
Unlike ADMM and Frank-Wolfe algorithm,
PG algorithms have rigorous convergence guarantee which ensures that a critical point of \eqref{eq:pro} can be produced \cite{gongZLHY2013,li2015accelerated}.
However,
their efficiency hinges on having a cheap closed-form solution of the proximal step
\begin{align}
\Px{\lambda g}{z} = \arg\min_x \frac{1}{2} \NM{x - z}{2}^2
+ \lambda g(x).
\label{eq:prox3}
\end{align} 
where $\lambda \ge 0$ and $z \in \R^d$ is an input vector.

Usually,
the closed-form solution of the proximal step 
needs
some special properties on $g$,
such as convexity (e.g., the $\ell_1$-norm \cite{efron2004least} and nuclear norm \cite{cai2010singular}) 
or separability among the dimensions of $x$ (e.g., the log-sum-penalty \cite{gongZLHY2013} and $\ell_p$-norm with $p \in (0,1)$ \cite{chartrand2008iteratively}).
However,
the $-\NM{x}{2}$ inside the $\DifReg$-regularizer \eqref{eq:difl12} makes it neither convex nor separable.
As a result,
previous approaches to develop the closed-form solution on proximal step all
fail for the $\DifReg$-regularizer.
Then, 
its proximal step has to be solved exactly using other iterative algorithms.
This gives rise to high iteration time complexity inside PG algorithm,
which in turn destroys the efficiency of the whole algorithm
making it no significant superiorities over DCA based algorithms.
Besides,
it remains unclear how the proximal step can be computed efficiently.

To make PG algorithm efficient on $\DifReg$ regularized learning problems,
in this paper, we
derive a cheap closed-form solution for the proximal step
of this regularizer. 
The key observation is that when $g$ is 
the $\DifReg$-regularizer, the optimal solution of proximal step \eqref{eq:prox3} cannot be zero unless $z = 0$.
Thus,
we can only consider the case when $z \neq 0$,
this helps to remove non-differentiability of the $\ell_2$-norm
(which is not differential only at the zero point).
Then,
by checking the first-order optimality condition of \eqref{eq:prox3},
a closed-form solution can be obtained.
We further extend 
the proposed solution
to low-rank matrix learning and total variation model.
With such a closed-form solution at hand,
we are ready to reuse existing PG algorithms for problem \eqref{eq:pro}.
Specifically,
we pick up the nonmonotone accelerated proximal
gradient (nmAPG) algorithm \cite{li2015accelerated} as it is the state-of-art PG algorithm.
We perform experiments on various tasks as 
on compressed sensing, matrix completion and image denoising.
The results show nmAPG armed with given closed-form solution is much faster than DCA based algorithms for $\DifReg$ regularization,
and the $\DifReg$-regularizer gives better performance than other nonconvex regularizers.

The rest of the paper is organized as follows.
Section~\ref{sec:rev} briefly reviews proximal gradient algorithms and the DCA algorithm for $\DifReg$ regularization, 
then Section~\ref{sec:proxstep}
derives
the closed-form proximal step of the 
$\DifReg$-regularizer.
Section~\ref{sec:app}
extends the proposed solution
to the low-rank matrix completion problem and total variation model.
Experimental results are shown in 
Section~\ref{sec:expts}, and  the last section gives some
concluding remarks
and discuss some possible future works.



\vspace{4px}
\noindent
\textbf{Notation:}
For a smooth function $f$, $\nabla f$ denotes its gradient.
If $f$ is convex but nonsmooth,
its subdifferential is
$\partial f(x) = \left\lbrace  s : f(y) \le f(x) + (y - x)^{\top} s \right\rbrace$.
For a scalar $x \in \R$, 
$\sign{x} = 1$ if $x > 0$, $-1$ if $x < 0$, and $0$ otherwise.
For a vector $x \in \R^d$, $\NM{x}{1} = \sum_{i = 1}^d |x_i|$ is its $\ell_1$-norm, 
and $\NM{x}{2} = ( \sum_{i = 1}^d x_i^2 )^{1/2}$ is the $\ell_2$-norm.
For a matrix $X \in \R^{m \times n}$ (assume that $m \le n$),
its SVD is
$X = U \Sigma V^{\top}$, where 
$\Sigma = \Diag{[\sigma_1, \dots, \sigma_m]}$ with $\sigma_1 \ge \cdots \ge \sigma_m$,
and $\NM{X}{*} = \sum_{i = 1}^m \sigma_i$ is the nuclear norm;
$\NM{X}{21} = \sum_{j = 1}^n \left(  \sum_{i = 1}^m X_{ij}^2 \right)^{1/2}$ is the matrix $(2,1)$-norm;
$\NM{X}{F} = ( \sum_{i = 1}^m \sum_{j = 1}^n X_{ij}^2 )^{1/2}$ is the Frobenious norm;
and $\NM{X}{1} = \sum_{i = 1}^m \sum_{j = 1}^n |X_{ij}|$ is the $\ell_1$-norm.
For a square matrix $X \in \R^{m \times m}$,
$\Tr{X} = \sum_{i = 1}^m X_{ii}$ denotes its trace.


\section{Review}
\label{sec:rev}

\subsection{Difference of Convex Algorithm}
\label{sec:dca}

As the $\DifReg$-regularizer can be naturally decomposed as 
a difference of two convex functions,
existing $\DifReg$ solvers 
\cite{yin2015minimization,lou2015weighted,matruncated2016}
are all based 
on the difference of convex algorithm (DCA) \cite{yuille2002concave}
(Algorithm~\ref{alg:dca}).
With $g$ in \eqref{eq:pro} being the $\DifReg$-regularizer,
at the $t$th iteration,
DCA replaces
the nonconvex $\DifReg$ by a linear approximation.
The next iterate 
$x_{t + 1}$
is generated from the following 
problem:
\begin{align}
\! x_{t + 1} \!
= \! \arg\min_{x} f(x) 
\! + \! \lambda \left( \NM{x}{1} - \NM{x_t}{2} 
\! -\!  (x - x_t)^{\top} s_t \right)
\!,  
\label{eq:dca}
\end{align} 
where $s_t \in \partial \NM{x_t}{2}$,
and
\begin{align}
s_t
= 
\begin{cases}
\frac{
x_t
}{\NM{x_t}{2}},
& x_t \neq 0
\\
\text{any}\; u \;\text{with}\; \NM{u}{2} \le 1,
&\text{otherwise}
\end{cases}.
\label{eq:subl2}
\end{align}
The algorithm converges to a critical point of \eqref{eq:pro} \cite{yuille2002concave}.
However, (\ref{eq:dca})
does not have a closed-form solution, and has to be solved by 
algorithms such as 
alternating direction method of multipliers (ADMM)
\cite{boyd2011distributed},
Frank-Wolfe algorithm \cite{jaggi2013revisiting} and Bregman iterative algorithm \cite{yin2008bregman}.
Hence, each iteration can be expensive and inefficient
\cite{gongZLHY2013,li2015accelerated,zhaosong2012}.

\begin{algorithm}[ht]
\caption{DCA for solving problem~\eqref{eq:pro}, with $g$ being the $\DifReg$-regularizer
\cite{yin2015minimization}.}
\begin{algorithmic}[1]
	\STATE{Initialize $x_1 = 0$;}
	\FOR{$t = 1, \dots, T$}
	\STATE $s_{t} \in \partial \NM{x_t}{2}$;
	\STATE $x_{t + 1} = \arg\min_x f(x) + \lambda \NM{x}{1} - \lambda (x^{\top} s_t)$;
	\ENDFOR
	\RETURN $x_{T + 1}$.
\end{algorithmic}
\label{alg:dca}
\end{algorithm}

When $f$ is $L$-Lipschitz smooth (i.e.,
$\NM{\nabla f(x) - \nabla f(y) }{2} \le L \NM{x - y}{2}$ for any $x,y$),
a more efficient DCA variant is the sequential convex programming (SCP) algorithm \cite{zhaosong2012}
(Algorithm~\ref{alg:scp}).
It 
generates $x_{t + 1}$ as
\begin{eqnarray}
x_{t + 1}
& = & \arg\min_x f(x_t) + (x - x_t)^{\top} \nabla f(x_t) + \frac{L}{2}\NM{x - x_t}{2}^2
\notag \\
& &+ \lambda \left( \NM{x}{1} - \NM{x_t}{2} - (x - x_t)^{\top} s_t \right)
\notag \\
&= & \Px{\frac{\lambda}{L} \NM{\cdot}{1}}{x_t + \frac{\lambda}{L} s_t - \frac{1}{L} \nabla f(x_t)},
\label{eq:scp}
\end{eqnarray}
where $s_t \in \partial \NM{x_t}{2}$.
The proximal step associated with the $\ell_1$-regularizer
has the closed-form solution 
$[ \Px{\lambda \NM{\cdot}{1}}{z} ]_i = \sign{z_i} \max(|z_i| - \lambda, 0)$ \cite{efron2004least}.
Thus,
each SCP iteration is cheap, and SCP is much faster than DCA.

\begin{algorithm}[H]
\caption{SCP for solving \eqref{eq:pro},
with $g$ being the $\DifReg$-regularizer
\cite{zhaosong2012}.}
\begin{algorithmic}[1]
	\STATE{Initialize $x_1 = 0$;}
	\FOR{$t = 1, \dots, T$}
	\STATE $s_{t} \in \partial \NM{x_t}{2}$;
	\STATE $x_{t + 1} = \Px{\frac{\lambda}{L} \NM{\cdot}{1}}{x_t + \frac{\lambda}{L} s_t - \frac{1}{L} \nabla f(x_t)}$;
	\ENDFOR
	\RETURN $x_{T + 1}$.
\end{algorithmic}
\label{alg:scp}
\end{algorithm}

However,
while the convergence of first-order optimization algorithms can be significantly improved
by acceleration
\cite{li2015accelerated,nesterov2013gradient},
SCP cannot be
accelerated.
Moreover, as $s_t$ is a subgradient,
the convergence of SCP 
can be slow 
near nonsmooth points \cite{beck2009fast}. 

\subsection{Proximal Gradient Algorithm}

The proximal gradient (PG) algorithm \cite{parikh2013proximal} has been commonly used for
solving optimization problems in the form of \eqref{eq:pro}.  Traditionally, both $f$ and $g$
are assumed to be convex, and $f$ is also required to be smooth.
At the $t$th iteration,
the next iterate $x_{t + 1}$ is generated as
\begin{align}
x_{t + 1}
= \Px{\frac{\lambda}{L} g}{x_t - \frac{1}{L} \nabla f(x_t)}.
\label{eq:pgiter}
\end{align}
The PG algorithm converges 
with a $O(1/T)$ rate, where $T$ is the number of iterations \cite{beck2009fast}.
Using Nesterov acceleration
\cite{beck2009fast,nesterov2013gradient},
\eqref{eq:pgiter} is slightly changed to
\begin{align}
y_t 
& = x_t + \frac{\alpha_{t - 1} - 1}{\alpha_t} \left( x_t - x_{t - 1} \right),
\label{eq:accpgiter1}
\\
x_{t + 1} 
& = \Px{\frac{\lambda}{L} g}{y_t - \frac{1}{L} \nabla f(y_t)},
\label{eq:accpgiter2}
\end{align}
where $\alpha_0 = \alpha_1 = 0$, 
and $\alpha_{t + 1} = \frac{1}{2}( \sqrt{4 \alpha_t^2 + 1} + 1)$, and
the resultant convergence rate is improved to $O(1/T^2)$.
In order for PG and its accelerated variant (APG) to be efficient,
the proximal step has to be cheap (e.g., has a closed-form solution).

Recently, the PG algoirthm (with its accelerated variant) have also been
extended to nonconvex problems
\cite{gongZLHY2013,li2015accelerated,bolte2014proximal,attouch2010proximal}, with
the state-of-the-art being the nonmonotone accelerated proximal
gradient (nmAPG) algorithm 
\cite{li2015accelerated}.
In the sequel,
we adopt nmAPG for learning with the $\DifReg$-regularizer.


\section{Proximal Step for $\ell_{1\text{-}2}$ Regularizer}
\label{sec:proxstep}

The following Proposition shows the
existence of the solution on proximal step.

\begin{proposition}[\cite{attouch2010proximal}] \label{pr:solext}
If $g$ is proper, lower semicontinuous, and $\inf g > -\infty$,
then the solution set of \eqref{eq:prox3} is nonempty and compact.
\end{proposition}

Note that $\inf_x \left( \NM{x}{1} - \NM{x}{2} \right)  \ge 0$,
as $\NM{x}{2} \le \NM{x}{1}$ for any $x \in \R^d$.
Thus, the $\DifReg$ regularization satisfies the assumptions in Proposition~\ref{pr:solext}. 
Let $z \in \R^d$ be an arbitrary vector, and 
$\phi(x) = \frac{1}{2}\NM{x - z}{2}^2 + \lambda( \NM{x}{1} - \NM{x}{2} )$,
then
\begin{align}
x^* = \text{prox}_{\lambda \NM{\cdot}{1\text{-}2}}(z)
=  \arg\min_x \phi(x),
\label{eq:prox}
\end{align}
must exist.
However,
no closed-form solution has been offered,
thus $x^*$ has to be obtained using iterative algorithms, e.g., DCA \cite{yin2015minimization},
which makes PG algorithms inefficient.

%
%

\subsection{Numerical Method}
\label{sec:dcaprox}

The closed-form solution of the proximal step \eqref{eq:prox} will be later derived at Section~\ref{sec:closedsol}.
To better illustrate the benefit of using such closed-form,
we first show how \eqref{eq:prox} can be solved numerically with DCA algorithm (Algorithm~\ref{alg:dca}).
Due to the special form of $\phi$,
step~4 in Algorithm~\ref{alg:dca} becomes
\begin{align}
x_{t + 1}
= \Px{\lambda \NM{\cdot}{1}}{z + \lambda s_t}.
\label{eq:pxdca2}
\end{align}
Thus,
\eqref{eq:pxdca2} has the closed-form solution.
If $t_p$ iterations are used by the DCA algorithm,
$O(d t_p)$ time is needed.

While SCP algorithm is generally more efficient than DCA algorithm,
again due to the special form of $\phi$,
\eqref{eq:pxdca2} is equivalent to SCP updates \eqref{eq:scp}.
Thus,
DCA is as efficient as SCP algorithm on solving \eqref{eq:prox}.


\subsection{Closed-form Solution}
\label{sec:closedsol}

To compute the proximal step, 
the regularizer 
is usually required
to be convex \cite{cai2010singular}
or separable among dimensions of $x$ \cite{chartrand2008iteratively,gongZLHY2013}.
However, this is not the case for the $\DifReg$-regularizer.
Our approach here is inspired by the following Lemma.

\begin{lemma} \label{lem:nzsol}
$x^* = 0$ iff $z$ in (\ref{eq:prox}) is equal to $0$.
\end{lemma}

\begin{proof}
\textit{Necessary condition}: Note that $\phi(x) \ge 0$ for any $x$,
and $\phi(0) = 0$ when $z = 0$.
Thus, if $z = 0$, the optimal solution is $x^* = 0$.
\textit{Sufficient condition}: 
If $x^* = 0$,
we show $z$ must be $0$.
In this case,
assume $z \neq 0$,
we can pick up an arbitrary non-zero dimension $z_j$ in $z$
and construct $\hat{x} \in \R^d$ as
$\hat{x}_i = 
\begin{cases}
0,
& i \neq j
\\
z_j, 
&
i = j
\end{cases}$.
Since $\NM{\hat{x}}{1} = \NM{\hat{x}}{2}$,
we have
\begin{align*}
\phi(x^*) 
= \phi(0) = \sum_{i = 1}^d z_i^2
> \sum_{i \neq j} z_i^2 = \phi(\hat{x}).
\end{align*}
This is in the contradictory with $x^*$ being optimal solution for $z \neq 0$.
Thus $x^* = 0$, we must have $z = 0$.
\end{proof}

Hence, in the sequel, we will only consider the case where $z \neq 0$ (and thus
$x^* \neq 0$).

\begin{lemma} \label{lem:opt}
If $x^*$ is the optimal solution to \eqref{eq:prox},
then $0 \in x^*- z + \lambda ( \partial \NM{x^*}{1} - \partial \NM{x^*}{2})$.
\end{lemma}

\begin{proof}
According to \cite{yuille2002concave},
we must have 
\begin{align*}
0 \in x - z + \lambda ( \partial \NM{x}{1} - \partial \NM{x}{2}).
\end{align*}
for any critical point $x$ of \eqref{eq:prox}.
Then,
the Lemma holds as all optimal points must also be critical points.
\end{proof}

As $x^* \neq 0$, 
the only non-differentiable point of the $\ell_2$-norm
is removed. 
Using the fact that
$\partial \NM{x}{1}  = [ a_i ]$ where
\begin{align*}
a_i
= 
\begin{cases}
1,
&
x_i > 0
\\
[-1, 1],
&
x_i = 0
\\
-1,
&
\text{otherwise}
\end{cases},
\end{align*}
and
	$\partial \NM{x}{2}$ in 
(\ref{eq:subl2}),
the condition in Lemma~\ref{lem:opt} can be simplified as
\begin{align}
0 \in x^* - z + \lambda \left( a
- \frac{x^*}{\NM{x^*}{2}}\right).
\label{eq:opt2}
\end{align}
This condition plays the key to finding the proposed closed-form solution.
The following Lemma
shows 
$\sign{x^*_i}$.

\begin{lemma}
	\label{lem:samesign}
	If $x^*_i \neq 0$,
	$\sign{x^*_i} = \sign{z_i}$.
\end{lemma}

\begin{proof}
	We prove it by establishing contradiction.
	Constructing a $\hat{x}$ such that
	$\hat{x}_j =
	\begin{cases}
	x^*_j & j \neq i
	\\
	- x^*_j & j = i
	\end{cases}$,
	where $\hat{x}$ has only one non-zero dimension with opposite sign as $x^*$.	
	Assuming $\hat{x}$ achieves optimal.
	First, we have
	\begin{align}
	& \sum_{j \neq i} (\hat{x}_j - z_j)^2 + (\hat{x}_i - z_i)^2
	= \sum_{j \neq i} (x^*_j - z_j)^2 + (\hat{x}_i - z_i)^2
	\notag \\
	& \quad\quad\quad > \sum_{j \neq i} (x^*_j - z_j)^2 + (x^* - z_i)^2,
	\label{eq:temp1}
	\end{align}
	where the last inequality dues to $x^*_i$ has the same sign as $z_i$.
	Then, note that
	\begin{align}
	\NM{x^*}{1} - \NM{x^*}{2} = \NM{\hat{x}}{1} - \NM{\hat{x}}{2}.
	\label{eq:temp2}
	\end{align}
	Combing \eqref{eq:temp1} and \eqref{eq:temp2}, we must have
	\begin{align*}
	\phi(\hat{x})
	& = 
	\frac{1}{2}\NM{\hat{x} - z}{2}^2 
	+ \lambda\left( \NM{\hat{x}}{1} - \NM{\hat{x}}{2} \right)
	\\
	& > 
	\frac{1}{2}\NM{x^* - z}{2}^2 
	+ \lambda\left( \NM{x^*}{1} - \NM{x^*}{2} \right)
	= \phi(x^*),
	\end{align*}
	which is in contrast with the assumption that $\hat{x}$ is optimal.
	Thus, if $x_i^* \neq 0$,
	$x^*_i$ must have same sign as $z_i$.
\end{proof}

Then,
we obtain the following closed-form solution of (\ref{eq:prox}).

\begin{proposition} \label{pr:proxsol}
Let $w = [w_i] \in \R^d$, with $w_i = \sign{z_i} \text{max} \left( |z_i| - \lambda, 0 \right)$.
\begin{enumerate}
\item[(i).] If $w = 0$, then $x^*_j = z_j$ for
$j = \arg\max_{i = 1, \dots, d} |z_i|$, and
0  otherwise; 

\item[(ii).] If $w \neq 0$, $x^* = (1 + \frac{\lambda}{\NM{w}{2}} ) w$.
\end{enumerate}
\end{proposition}

\begin{proof}
As in Lemma~\ref{lem:samesign}, 
$x^* \neq 0$ and $x_i^*$ must have same sign with $z_i$, 
condition \eqref{eq:opt2} can be expressed as
\begin{align}
x^*_i - z_i + \lambda \sign{z_i} - \lambda \frac{x^*_i}{\NM{x}{2}} = 0, 
\quad & x^*_i \neq 0.
\label{eq:optnnz}
\\
- \lambda \le z_i \le \lambda, 
\quad & x^*_i = 0.
\label{eq:optzero}
\end{align}
Therefore, 
we can partition dimensions of $x$ into two set based on 
\eqref{eq:optnnz} (denoted as $\mathcal{N}$) and \eqref{eq:optzero} (denoted as $\mathcal{Z}$).
Let $\bar{x}$ be part of $x^*$ containing dimensions in $\mathcal{N}$ (accordingly for $\bar{z}$ from $z$), 
note that $\NM{\bar{x}}{2} = \NM{x}{2}$,
\eqref{eq:optnnz} can be expressed as
\begin{align}
\bar{x} \left( 1 - \frac{\lambda}{\NM{\bar{x}}{2}} \right) = \bar{z} - \lambda \sign{\bar{z}}.
\label{eq:temp3}
\end{align}
Then we can determine $\NM{\bar{x}}{2}$ from
$ \left| 1 - \frac{\lambda}{\NM{\bar{x}}{2}} \right|  \NM{\bar{x}}{2}
= \NM{\bar{z} - \lambda \sign{\bar{z}}}{2}$.
Let $c_z = \NM{\bar{z} - \lambda \sign{\bar{z}}}{2}$, then
\begin{align*}
\NM{\bar{x}}{2} = 
\begin{cases}
\lambda + c_z,
& c_z \ge \lambda
\\
\lambda - c_z \text{\;or\;} \lambda + c_z,
& \text{otherwise}
\end{cases}.
\end{align*}
However, as $\bar{x}_i$ must have same sign with $\bar{z}_i$,
this indicates $\lambda / \NM{\bar{x}}{2}$ in \eqref{eq:temp3} must be smaller than $1$.
Thus, $\NM{\bar{x}}{2} = \lambda + c_z$; and then the optimal solution can be expressed as
\begin{align}
x_i = \frac{c_z + \lambda}{c_z} 
\left[ \sign{z_i} \max\left( |z_i| - \lambda, 0 \right) \right].
\label{eq:temp4}
\end{align}
Note that if there exists dimensions in $z$ such that $|z_i| > \lambda$ holds,
it can be seen there is only one point satisfying the necessary conditions \eqref{eq:optnnz} and \eqref{eq:optzero}. 
Thus, this point must be the global optimal.
Then, 
if $z_i \le \lambda$ for all dimensions, \eqref{eq:optzero} shows $x^* = 0$,
which violates Lemma~\ref{lem:nzsol},
and we must have at least one non-zero dimension in $x^*$.
Let us consider following two cases
\begin{itemize}
	\item 
	If there is only one non-zero dimension, 
	then following Lemma~\ref{lem:nzsol},
	\begin{align}
	x^*_i = 
	\begin{cases}
	0, & i \neq j
	\\
	z_j, & i = j
	\end{cases}
	\;\text{where}\;
	j = \arg\max_{i = 1, \cdots, d} |z_i|.
	\label{eq:temp12}
	\end{align}
	
	\item If there are more than one non-zero dimensions,
	we will go back to \eqref{eq:temp3},
	which shows $x^* = 0$.
	Thus, we can only have one non-zero dimension.
\end{itemize}
Finally, when $z = 0$, then $x^* = 0$ and is included in \eqref{eq:temp12}.
The Proposition is then obtained from \eqref{eq:temp4} and \eqref{eq:temp12}.
\end{proof}

For the DCA algorithm,
$O(d t_{p})$ time is needed to compute $x^*$
where $t_p$ is the number of iterations.
Instead,
using Proposition~\ref{pr:proxsol},
it takes only $O(d)$ time,
which can be much faster.



\section{Extensions}
\label{sec:app}

In this section, 
we extend Proposition~\ref{pr:proxsol}
to
low-rank matrix learning
(Section~\ref{sec:matcomp}) and the total variation model (Section~\ref{sec:tvmdl}).

\subsection{Low-Rank Matrix Completion}
\label{sec:matcomp}

In low-rank matrix completion, one tries to recover a low-rank matrix from a small number of
observations \cite{candes2009exact}.  
Let matrix $O \in \R^{m \times n}$
with observed positions indicated by $\Omega \in \{ 0, 1 \}^{m \times n}$
such that $\Omega_{ij} = 1$ if $O_{ij}$ is observed and $0$ otherwise.
The matrix completion problem is formulated as
\begin{align}
\min_{X \in \R^{m \times n}} \frac{1}{2} \NM{\SO{\Omega}{X - O}}{F}^2
+ \lambda g(X),
\label{eq:matcomp}
\end{align}
where 
$\left[\SO{\Omega}{A}\right]_{ij} = A_{ij}$ if $\Omega_{ij} = 1$
and $0$ otherwise,
and $g$ is a low-rank regularizer.
Two common choices of $g$ are the nuclear norm \cite{candes2009exact,cai2010singular} and rank constraint \cite{jain2010guaranteed,wen2012solving}.
Matrix completion has been successfully applied to many applications
such as recommender system \cite{jain2010guaranteed,wen2012solving} and image recovery \cite{lu2016nonconvex,wang2015orthogonal}.

Let the singular values of $X$ be $\sigma \equiv [\sigma_1, \dots, \sigma_m]$ 
(arranged in nonincreasing order).
Recall that the nuclear 
and Frobenious norms of $X$ 
are $\NM{X}{*} = \NM{\sigma}{1}$ 
and $\NM{X}{F} = \NM{\sigma}{2}$, respectively.
Thus, they can be viewed as the $\ell_1$- and $\ell_2$-norms
of the singular values \cite{candes2009exact}.
Based on this  observation, the $\DifReg$-regularizer has been recently extended to matrices as \cite{matruncated2016}:
\begin{align}
\NM{X}{*\text{-}F} = \NM{X}{*} - \NM{X}{F}.
\label{eq:l12low}
\end{align}
Note that
$\NM{\cdot}{*\text{-}F}$
is nonconvex and nonsmooth.
In \cite{matruncated2016},
reliable recovery guarantee for matrix completion is provided with the use of
the $\NM{\cdot}{*\text{-}F}$ regularizer,
and better empirical performance than both norm regularization  and explicit rank constraint is also observed. 
However, 
DCA 
is still used 
in \cite{matruncated2016}, and is slow.

In this paper, 
we propose the use of the PG algorithm.
It has been demonstrated great success in low-rank learning with
the nuclear norm regularization \cite{hsieh2014nuclear,qyao2015ijcai,mazumder2010spectral},
rank constraint \cite{jain2010guaranteed},
and some adaptive nonconvex regularization such as the LSP function
and capped $\ell_1$-norm \cite{qyao2015icdm,lu2016nonconvex}.
However,
it has not been used with the $\DifReg$ regularization yet.

The following Proposition shows that the 
proximal step 
associated with
$\NM{\cdot}{*\text{-}F}$ 
can also be efficiently computed.

\begin{proposition} \label{pr:proxsol2}
Let 
the SVD 
of a given $Z \in \R^{m \times n}$
be $U \Sigma V^{\top}$, 
$X^* = \Px{\lambda \NM{\cdot}{*\text{-}F}}{Z}$, and $w 
= [w_i] 
\in \R^m$ with $w_i = \text{max} \left( \sigma_i - \lambda, 0 \right)$. 
\begin{enumerate}
	\item[(i).] If $w = 0$, $X^* = \sigma_1 u_1 v_1^{\top}$;
	\item[(ii).]  If $w \neq 0$, $X^* = U \Diag{(1 + \frac{\lambda}{\NM{w}{2}}) w} V^{\top}$.
\end{enumerate}
\end{proposition}

\begin{proof}
Let the SVD of $X$ be $\bar{U} \Diag{\bar{\sigma}} \bar{V}^{\top}$
where $\bar{\sigma} = [\bar{\sigma}_1, \dots, \bar{\sigma}_m]$, 
then \eqref{eq:l12low} becomes
\begin{align}
\arg\min_{X}
\frac{1}{2}\NM{X - Z}{F}^2 
+ \lambda\left( \NM{\bar{\sigma}}{1} - \NM{\bar{\sigma}}{2} \right).
\label{eq:temp8}
\end{align}
For the first term in \eqref{eq:temp8}, we have
\begin{align*}
\min_X \frac{1}{2}\NM{X - Z}{F}^2
= \min_X \frac{1}{2} \left( \NM{\sigma}{2}^2 + \NM{\bar{\sigma}}{2}^2 \right)
- \Tr{X^{\top} Z}.
\end{align*}
Note that $\Tr{X^{\top} Z} \le \sum_{i = 1}^m \sigma_i \bar{\sigma}_i$ and
the equality achieves only when $\bar{U} = U$ and $\bar{V} = V$ \cite{golub2012matrix}.
Thus, \eqref{eq:temp8} becomes 
\begin{align*}
\arg\min_{\bar{\sigma}}
\frac{1}{2}\NM{\bar{\sigma} - \sigma}{2}^2 
+ \lambda\left( \NM{\bar{\sigma}}{1} - \NM{\bar{\sigma}}{2} \right),
\end{align*}
where $\bar{\sigma}$ can be obtained from Proposition~\ref{pr:proxsol}.
\end{proof}

Given the SVD of $Z$, computing the proximal step takes only 
$O(m)$ time.
However,
a direct SVD computation takes $O(m^2 n)$ time.
As noted in
\cite{qyao2015ijcai,mazumder2010spectral}, during execution of the proximal
algorithm,
$Z$ is the sum of a sparse matrix and a low-rank matrix.
This can be used to speedup matrix multiplications in SVD computation, 
thus reducing the time
complexity to 
$O( k^2(m + n) + k \NM{\Omega}{1} )$ time.
As a result, 
the proximal step takes only $O(  k^2(m + n) + k \NM{\Omega}{1} )$ time,
where $k \ll m$ is the desired rank of the output matrix.

\subsection{Total Variation Model}
\label{sec:tvmdl}

The total variation (TV) model has been commonly used in image processing \cite{osher2005AnIR}.
Let $x \in \R^{mn}$ be a vectorized image of size 
$m \times n$,
and $d = mn$.
The TV regularizer
is defined as
$\text{TV}(x)
= \NM{\D_h x}{1} + \NM{\D_v x}{1}$,
where $\D_h \in \R^{d \times d}$ and $\D_v \in \R^{d \times d}$ are the horizontal and vertical partial derivative operators,
respectively.
Recently, 
the $\DifReg$-regularizer has been extended to TV regularization
\cite{lou2015weighted}:
\begin{eqnarray} 
\TV{x}
& = & \NM{\D_h x}{1} + \NM{\D_v x}{1} \nonumber \\
&&-
\sum_{i = 1}^d \sqrt{ [\D_h x]_i^2 + [\D_v x]_i^2}.
\label{eq:newtv}
\end{eqnarray}
It has outperformed
standard $\ell_0$ and $\ell_1$-regularizers
on many tasks such as image denoising, image deblurring 
and MRI reconstruction \cite{lou2015weighted}.
Again, 
$\text{TV}_{1\text{-}2}$
is nonconvex and nonsmooth.

Given an input corrupted image $y$,
for above applications
the 
image $x$ 
can be 
recovered 
as 
\begin{align}
\min_{x} \frac{1}{2}\NM{A x - y}{2}^2 + \lambda \TV{x}.
\label{eq:tvmdl}
\end{align}
where matrix $A$ depends on the specific choice of application.
We consider image denoising here so that $A = I$.
Again,
DCA is 
used in \cite{lou2015weighted}.

In the following, we propose a more efficient approach for learning with \eqref{eq:newtv}.
First, we rewrite (\ref{eq:newtv})
to a different form by extending
Proposition~\ref{pr:proxsol}.

\begin{lemma}
$\TV{x} = \NM{\mathcal{D}(x)}{1\text{-}(2,1)}$,
where $\mathcal{D}(x) \equiv [\D_h x, \D_v x]$, 
and $\NM{X}{1\text{-}(2,1)} \equiv \NM{X}{1} - \NM{X}{2,1}$.
\end{lemma}

\begin{proof}
We can write $\NM{\mathcal{D}(x)}{1\text{-}(2,1)}$ as
$\sum_i^d
\left| \left[ \D_h x \right]_i \right| 
+ \left| \left[ \D_v x \right]_i \right| 
\!-\! \sqrt{\left[ \D_h x \right]^2_i
\!+\! \left[ \D_v x \right]^2_i}$,
which is equivalent
	to \eqref{eq:newtv}.
\end{proof}

Instead of \eqref{eq:tvmdl},
we consider the optimization problem
\begin{equation}
\label{eq:temp9}
\min_{x, W} \frac{1}{2}\NM{x - y}{2}^2
+ \lambda \NM{W}{1\text{-}(2,1)}
+ \frac{\mu}{2}\NM{W - \mathcal{D}(x)}{F}^2,
\end{equation}
where $\mu>0$ is a penalty parameter (in the experiment, we simply set $\mu = 100 \lambda$).
Though \eqref{eq:temp9}
is slightly different from
(\ref{eq:tvmdl}), experiments 
in Section~\ref{sec:imgrec} show that 
they have comparable recovery performance.

The following shows that 
(\ref{eq:temp9})
can be 
efficiently solved by alternating minimization.

\subsubsection{$x$ update}
At the $t$th iteration,
with a fixed $W_{t}$,
\begin{eqnarray}
x_{t + 1} & = & \arg\min_{x} \frac{1}{2}\NM{x - y}{2}^2
+ \frac{\mu}{2}\NM{W_t - \mathcal{D}(x)}{F}^2
\notag \\
& = & B^{-1} ( y + \mu \D_h^{\top} w^h_t + \mu \D_v^{\top} w^v_t ), \label{eq:B}
\end{eqnarray}
where $B = \mu \D_h^{\top} \D_h + \mu \D_v^{\top} \D_v + I$, and $I$ is the identity matrix.
Directly inverting $B$ takes $O(d^3)$ time which is expensive.
Instead, 
we 
use conjugate gradient descent (CGD) to  solve:
\begin{align}
B x_{t+1} =  ( y + \mu \D_h^{\top} w^h_t + \mu \D_v^{\top} w^v_t ).
\label{eq:clsx}
\end{align}
In each CGD iteration,
the most expensive step is the multiplications of $B u$, where $u \in \R^d$.
This can be rewritten as
$B u = \mu \D_h^{\top} (\D_h u) + \mu \D_v^{\top} (\D_v u) + u$.
As $\D_v, \D_h$ are partial derivative operators,
for any vector $v \in \R^d$,
$\D_h v, \D_v v,
\D_h^{\top} v$ and $\D_v^{\top} v$
can be computed in $O(d)$ time.
Besides, we use $x_t$ to warm-start CGD on solving $x_{t + 1}$.
Due to the fast convergence of CGD both in theory and practice \cite{nocedal2006},
a few iterations are enough.
Thus, $x_{t+1}$ in (\ref{eq:clsx}) can be obtained in $O(d)$ time.

\subsubsection{$W$ update}

With a fixed $x_{t+1}$,
\begin{eqnarray} 
W_{t + 1} & = & \arg\min_{W} \frac{\mu}{2}\NM{W - \mathcal{D}(x_{t + 1})}{F}^2 + \lambda
\NM{W}{1\text{-}(2,1)} \nonumber\\
& = & \Px{\frac{\lambda}{\mu} \NM{\cdot}{1\text{-}(2,1)}}{\mathcal{D}(x_{t + 1})}.  \label{eq:proxtv}
\end{eqnarray} 
The following shows that 
this proximal step 
has a closed-form solution involving
$\Px{\lambda \NM{\cdot}{1\text{-}2}}{\cdot}$.
Recall that computing 
$\Px{\lambda \NM{\cdot}{1\text{-}2}}{\cdot}$ takes
$O(d)$ time, and
$W \in \R^{d \times 2}$. Hence,
$W_{t + 1}$
can also be obtained in $O(d)$ time.

\begin{proposition} \label{lem:proxTV}
Let $X^* = \Px{\lambda \NM{\cdot}{1\text{-}(2,1)}}{Z}$.
Then, 
$x^i = \Px{\lambda \NM{\cdot}{1\text{-}2}}{z^i}$, 
where  $x^i, z^i$ are the $i$th row of $X^*$ and $Z$, respectively.
\end{proposition}

\begin{proof}
We can write the proximal step as
\begin{align*}
& \Px{\lambda \NM{\cdot}{1\text{-}2}}{Z}\\
& = \arg\min_{X} \frac{1}{2}\NM{X - Z}{F}^2 
+ \lambda\left( \NM{X}{1} - \NM{X}{21} \right) 
\\
& = \arg \min_{\{ x^i \}}
\sum_{i = 1}^d \frac{1}{2} \NM{x^i - z^i}{2}^2
+ \lambda \left(  \NM{x^i}{1} - \NM{x^i}{2} \right).
\end{align*}
Note that in the last line, minimization w.r.t $x^i$'s are independent with each other,
and its optimal solution is given by Proposition~\ref{pr:proxsol}.
\end{proof}


In summary, each 
iteration 
of the alternating minimization algorithm
takes only $O(d)$ time,  
thus is very efficient.
The whole procedure is shown in Algorithm~\ref{alg:altmin}.
Its convergence has been shown for problems of the
form $\min_{x, W} h(x,W) \equiv f(x, W) + g(W) + r(x)$,
where $f$ is Lipschitz-smooth, $g, r$ are proper and lower semicontinuous,
and $\inf_{x,W} h > -\infty$ \cite{attouch2010proximal}.
It is easy to see that 
these assumptions hold
for \eqref{eq:temp9}.

\begin{algorithm}[ht]
\caption{Alternating minimization for \eqref{eq:temp9} (AltMin).}
\begin{algorithmic}[1]
\STATE Initialize $W_1 = 0$;
\FOR{$t = 1, \dots, T$}
\STATE
compute $x_{t + 1}$ from \eqref{eq:clsx} using CGD; // $x$ update

\STATE $W_{t + 1} = \Px{\frac{\lambda}{\mu} \NM{\cdot}{1\text{-}(2,1)}}{\mathcal{D}(x_{t +
1})}$ using 
Proposition~\ref{lem:proxTV}; // $W$ update
\ENDFOR
\RETURN $x_{T + 1}$.
\end{algorithmic}
\label{alg:altmin}
\end{algorithm}

Alternatively,
one may 
rewrite \eqref{eq:tvmdl} 
equivalently as
\begin{align*}
\min_{x} \frac{1}{2}\NM{x - y}{2}^2 + \lambda\NM{[w^h, w^v]}{1\text{-}(2,1)} :
\begin{bmatrix}
w^h \\ w^v
\end{bmatrix}
=  \begin{bmatrix}
\D_h \\ \D_v
\end{bmatrix} x,
\end{align*}
and then solve it with ADMM \cite{boyd2011distributed}.
However,
as $[\D_h, \D_v]^{\top} \in \R^{2d \times d}$,
the full row-rank assumption in
\cite{li2015global}
fails,
thus 
ADMM 
may not converge.


\section{Experiments}
\label{sec:expts}

In this section, we perform experiments on both synthetic (Section~\ref{sec:syn}) and
real-world data sets (Sections~\ref{sec:imgcomp} and \ref{sec:imgrec}) in a number of
applications.

\begin{table*}[ht]
	\centering
	\caption{CPU time (in seconds) on the compressed sensing data set, with
		$\lambda = 0.01 \times {0.25}^{i} $.
		The fastest and comparable algorithms (according to the pairwise t-test with 95\%
		confidence) are highlighted.}
	\vspace{-8px}
	\begin{tabular}{cc | c | c| c | c | c} \hline
		\multicolumn{2}{c|}{CPU time (sec)}& $i$ = 0              & $i$ = 1              & $i$ = 2              & $i$ = 3              & $i$ = 4               \\ \hline
		\multicolumn{2}{c|}{DCA}                  & 1.9$\pm$0.2          & 4.6$\pm$0.5          & 13.8$\pm$1.3         & 51.2$\pm$4.9         & 193.9$\pm$21.1        \\ \hline
		\multicolumn{2}{c|}{SCP}                  & 2.1$\pm$0.2          & 8.0$\pm$1.1          & 32.1$\pm$4.1         & 120.6$\pm$12.7       & 335.6$\pm$28.3        \\ \hline
		\multirow{2}{*}{nmAPG} &            numerical             & 1.5$\pm$0.1          & 3.0$\pm$0.1          & 7.1$\pm$0.3          & 12.6$\pm$0.5         & 24.2$\pm$0.9          \\ \cline{2-7}
		&           closed-form            & \textbf{0.9$\pm$0.1} & \textbf{1.6$\pm$0.1} & \textbf{3.5$\pm$0.3} & \textbf{6.2$\pm$0.4} & \textbf{11.7$\pm$0.7} \\ \hline\hline
		\multicolumn{2}{c|}{FISTA}                 & \textbf{0.9$\pm$0.2} & 2.1$\pm$0.5        & 7.7$\pm$1.6         & 28.9$\pm$5.5          & 72.1$\pm$10.3           \\ \hline
	\end{tabular}
	\label{tab:synspeed:time}
\end{table*}

\begin{table*}[ht]
	\centering
	\caption{Recovered RMSE (scaled by ${10}^{-2}$) on the compressed sensing data set,
		with
		$\lambda = 0.01 \times {0.25}^{i} $.
		The lowest and comparable algorithms (according to the pairwise t-test with 95\%
		confidence) are highlighted.}
	\vspace{-8px}
	\begin{tabular}{cc | c | c| c | c | c}
		\hline
		      \multicolumn{2}{c|}{RMSE ($\times {10}^{-2}$)}         & $i$ = 0                 & $i$ = 1                & $i$ = 2                & $i$ = 3                & $i$ = 4                \\ \hline
		      \multicolumn{2}{c|}{DCA}       & 15.54$\pm$1.89          & 5.01$\pm$0.65          & \textbf{3.83$\pm$0.54} & 4.52$\pm$1.15          & 5.21$\pm$2.35          \\ \hline
		      \multicolumn{2}{c|}{SCP}       & 31.28$\pm$6.59          & 10.46$\pm$4.19         & 7.04$\pm$2.41          & 7.95$\pm$3.88          & 14.39$\pm$11.75        \\ \hline
		\multirow{2}{*}{nmAPG} &  numerical  & \textbf{15.22$\pm$1.73} & \textbf{4.85$\pm$0.62} & \textbf{3.82$\pm$0.72} & \textbf{4.01$\pm$0.35} & \textbf{4.12$\pm$0.35} \\ \cline{2-7}
		                       & closed-form & \textbf{15.29$\pm$1.28} & \textbf{4.85$\pm$0.60} & \textbf{3.78$\pm$0.37} & \textbf{4.00$\pm$0.38} & \textbf{4.11$\pm$0.39} \\ \hline\hline
		     \multicolumn{2}{c|}{FISTA}      & 20.42$\pm$3.88          & 6.55$\pm$1.31          & 4.82$\pm$0.88          & 5.74$\pm$2.04          & 13.98$\pm$12.16        \\ \hline
	\end{tabular}
	\label{tab:synspeed:rmse}
\end{table*}

\begin{figure*}[ht]
	\centering
	\subfigure[$\lambda = 0.01 \times {0.25}^{0}$.]
	{\includegraphics[width = 0.32\textwidth]{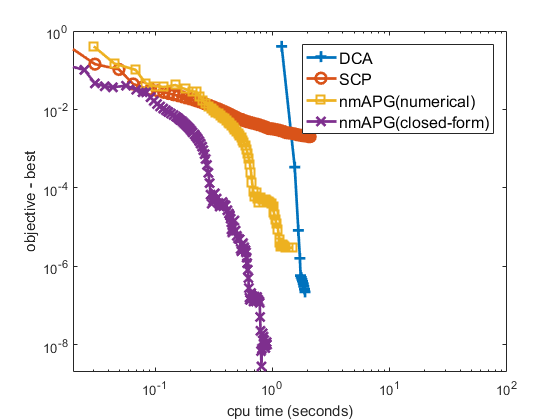}}
	\subfigure[$\lambda = 0.01 \times {0.25}^{2}$.]
	{\includegraphics[width = 0.32\textwidth]{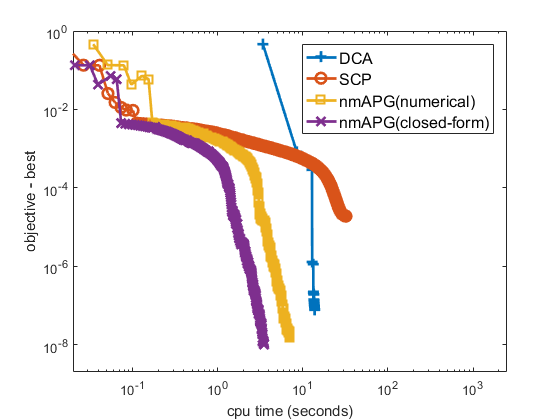}}
	\subfigure[$\lambda = 0.01 \times {0.25}^{4}$.]
	{\includegraphics[width = 0.32\textwidth]{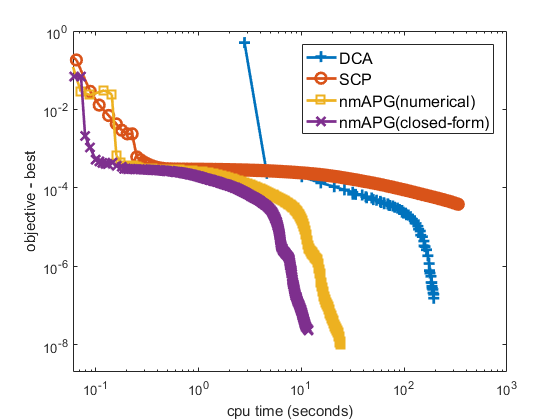}}
	
	\caption{Objective vs CPU time (in seconds) on the compressed sensing data set.
		Due to lack of space the figures for $\lambda = 0.01 \times 0.25^{1}$ and $\lambda = 0.01 \times 0.25^{3}$ are not shown.}
	\label{fig:synspeed:time}
	\vspace{-10px}
\end{figure*}

\subsection{Compressed Sensing}
\label{sec:syn}

Compressed sensing can be formulated as the following optimization problem 
\cite{yin2015minimization}:
\begin{equation} \label{eq:cs}
\min_{x} \frac{1}{2}\NM{y - A x}{2}^2 + \lambda ( \NM{x}{1} - \NM{x}{2} ),
\end{equation}
where 
$A$ is the dictionary,
$x$ is the sparse vector to be recovered,
and 
$y$ 
is the input noisy signal.
The task here is to recover the underneath sparse vector $x$ from a noisy measurement $y$ based on the given dictionary $A$.

We follow the setup in \cite{yin2015minimization}.
The data are generated as
$y = A \tilde{x} + \epsilon$,
where  $\tilde{x} \in \R^{4d}$ is sparse (with only $5\%$ of its entries nonzero, which are sampled i.i.d.
from the normal distribution
$\mathcal{N}(0, 1)$), 
$A \in \R^{d \times 4d}$ 
is an oversampled partial DCT matrix
(each element in its $i$th column $a_i$ is obtained as
$\frac{1}{\sqrt{{d}}} \cos\left( \frac{2 i \pi \varepsilon}{20} \right)$, where $\varepsilon$ is sampled from
the uniform distribution
$\mathcal{U}(0,1)$), and $\epsilon \in \R^{d}$ is the random noise sampled from $\mathcal{N}(0, 0.01)$.
Note that the dictionary $A$ is ill-conditioned \cite{yin2015minimization}.
For performance evaluation,
we use the normalized root-mean-squared error (RMSE):
$\frac{\NM{x - \tilde{x}}{2}}{\NM{x}{2}}$.
We set $d = 500$ and $\lambda \in 0.01 \times \{ 0.25^{0}, 0.25^{1}, \dots, 0.25^{4} \}$.  
The experiment is repeated $10$ times.

The following
algorithms are compared:
\begin{enumerate}
	\item 
	difference of convex  algorithm
	(DCA) \cite{yin2015minimization}
	(Algorithm~\ref{alg:dca}),
	and the subproblem is solved with ADMM \cite{boyd2011distributed};
	
	\item Sequential convex programming (SCP) \cite{zhaosong2012} (Algorithm~\ref{alg:scp});
	
	\item nmAPG 
	\cite{li2015accelerated}: The proximal step of the
	$\DifReg$-regularizer
	is computed in two ways: (i)
	numerically 
	using DCA (Algorithm~\ref{alg:dca} with warm-start) as Section~\ref{sec:dcaprox}
	(denoted ``nmAPG(numerical)'');
	and (ii) exactly
	using the closed-form solution in 
	Proposition~\ref{pr:proxsol}
	(denoted ``nmAPG(closed-form)'').
\end{enumerate}
As an additional baseline,
we also perform
$\ell_1$-norm regularization 
using FISTA \cite{beck2009fast}.

Table~\ref{tab:synspeed:time} shows the timing results of different algorithms.
As can be seen, DCA is the slowest;
nmAPG(closed-form) is always faster than nmAPG(numerical)
as the proximal step does not need to be solved iteratively.
A more detailed comparison on convergence of the objective is shown in Figure~\ref{fig:synspeed:time}.
We can see that nmAPG(closed-form) converges very quickly,
and it is even faster than 
FISTA for the convex problem.  
However, DCA and SCP suffer from slow convergence,
and SCP converges with a larger objective value.
Table~\ref{tab:synspeed:rmse} shows the recovery performance.
As can be seen, $\DifReg$ regularization consistently achieves lower RMSE than the convex
$\ell_1$-regularizer.

\begin{table*}[ht]
	\centering
	\caption{Performance on the image completion problem,
		the RMSE is scaled by $10^{-2}$.
		The best and comparable results (according to the pairwise t-test with 95\% confidence) are highlighted.}
	\label{tab:imgcomp}
	\vspace{-8px}
	\begin{tabular}{c|c|c|c|c|c|c|c|c}
		\hline
		\multicolumn{2}{c|}{}                & \multicolumn{2}{c|}{nuclear norm regularizer} &         \multicolumn{3}{c|}{factorization approaches}         & \multicolumn{2}{c}{nonconvex regularizer} \\
		\multicolumn{2}{c|}{}                &  AIS-Impute   &     SktechCG      &       ARSS-M3F       &    LMaFit     &        ER1MP         &     FaNCL     &           nmAPG           \\ \hline
		\multirow{2}{*}{\textit{Mountain}} &      RMSE      & 1.75$\pm$0.02 &   2.53$\pm$0.62   &    2.44$\pm$0.02     & 2.41$\pm$0.02 &    2.55$\pm$0.03     & 1.57$\pm$0.01 &  \textbf{1.41$\pm$0.02}   \\ \cline{2-9}
		& CPU time (sec) &  3.7$\pm$0.1  &    5.8$\pm$0.1    & \textbf{1.5$\pm$0.2} &  6.2$\pm$0.1  & \textbf{1.6$\pm$0.1} & 41.2$\pm$0.7  &        7.0$\pm$0.3        \\ \hline
		\multirow{2}{*}{\textit{Windows}}  &      RMSE      & 1.84$\pm$0.03 &   2.77$\pm$0.43   &    2.48$\pm$0.01     & 2.46$\pm$0.01 &    2.77$\pm$0.10     & 1.67$\pm$0.06 &  \textbf{1.54$\pm$0.06}   \\ \cline{2-9}
		& CPU time (sec) &  4.4$\pm$0.1  &    5.6$\pm$0.1    & \textbf{1.4$\pm$0.2} &  5.7$\pm$0.3  & \textbf{1.5$\pm$0.1} & 60.9$\pm$1.2  &        9.2$\pm$0.2        \\ \hline
		\multirow{2}{*}{\textit{Sea}}    &      RMSE      & 0.98$\pm$0.01 &   1.62$\pm$0.13   &    1.65$\pm$0.10     & 1.36$\pm$0.01 &    1.49$\pm$0.02     & 0.95$\pm$0.01 &  \textbf{0.81$\pm$0.01}   \\ \cline{2-9}
		& CPU time (sec) &  3.8$\pm$0.1  &    5.5$\pm$0.1    & \textbf{0.4$\pm$0.1} &  7.7$\pm$0.4  &     1.5$\pm$0.1      & 40.4$\pm$1.2  &        7.6$\pm$0.1        \\ \hline
	\end{tabular}
\end{table*}

\begin{figure*}[ht]
	\centering
	
	\includegraphics[height = 0.12\textwidth]{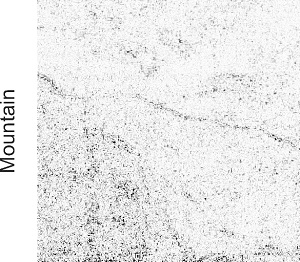}
	\includegraphics[height = 0.12\textwidth]{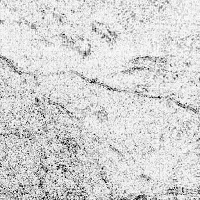}
	\includegraphics[height = 0.12\textwidth]{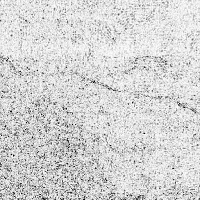}
	\includegraphics[height = 0.12\textwidth]{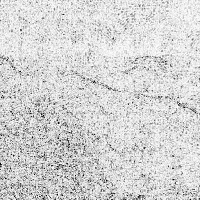}
	\includegraphics[height = 0.12\textwidth]{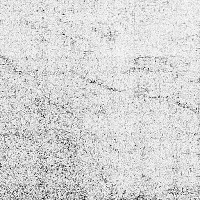}
	\includegraphics[height = 0.12\textwidth]{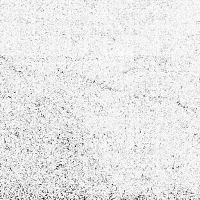}
	\includegraphics[height = 0.12\textwidth]{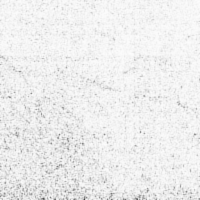}
	
	\vspace{6px}
	
	\includegraphics[height = 0.12\textwidth]{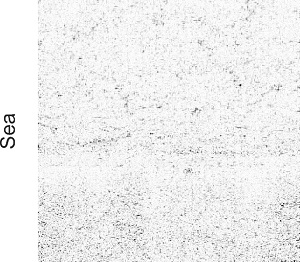}
	\includegraphics[height = 0.12\textwidth]{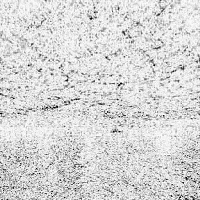}
	\includegraphics[height = 0.12\textwidth]{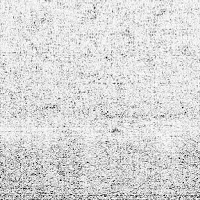}
	\includegraphics[height = 0.12\textwidth]{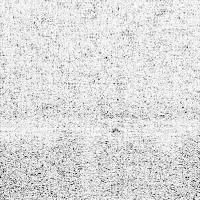}
	\includegraphics[height = 0.12\textwidth]{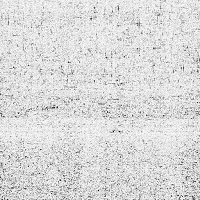}
	\includegraphics[height = 0.12\textwidth]{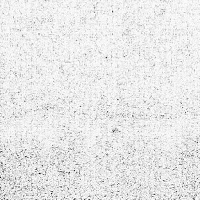}
	\includegraphics[height = 0.12\textwidth]{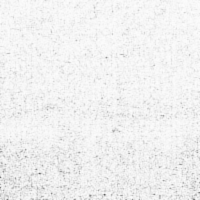}

	\subfigure[AIS-Impute.]
	{\includegraphics[height = 0.12\textwidth]{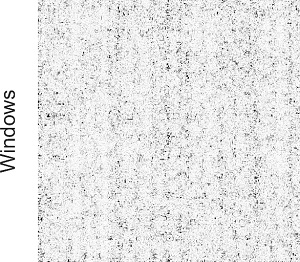}}
	\subfigure[SktechCG.]
	{{\includegraphics[height = 0.12\textwidth]{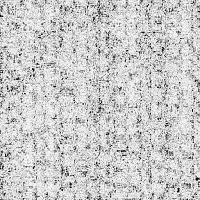}}}
	\subfigure[ARSS-M3F.]
	{\includegraphics[height = 0.12\textwidth]{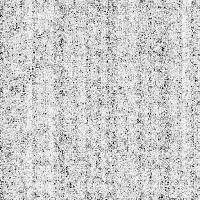}}
	\subfigure[LMaFit.]
	{\includegraphics[height = 0.12\textwidth]{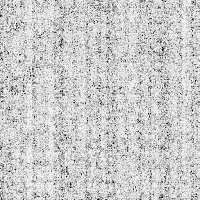}}
	\subfigure[ER1MP.]
	{\includegraphics[height = 0.12\textwidth]{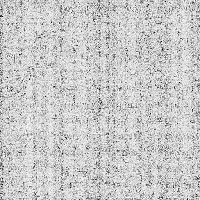}}
	\subfigure[FaNCL.]
	{\includegraphics[height = 0.12\textwidth]{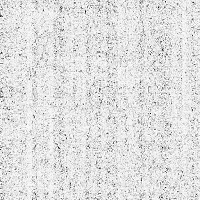}}
	\subfigure[nmAPG.]
	{\includegraphics[height = 0.12\textwidth]{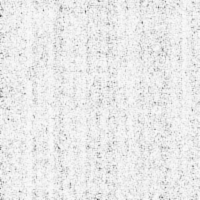}}
	
	\caption{Difference of recovered images from different algorithms to the clean images.
	The images ``\textit{Mountain}'', ``\textit{Sea}'' and ``\textit{Windows}'' are at the first, second and third row respectively. 
	The darker pixels indicate larger difference.}
	\label{fig:imgcomp}
\end{figure*}

\subsection{Image Completion}
\label{sec:imgcomp}

In this section,
we perform experiments on the matrix completion problem \eqref{eq:matcomp}.
We use 
three 
$512 \times 512$
gray-scale images (Figure~\ref{fig:clnimg}) from \cite{lu2016nonconvex}.
The pixel values are normalized to $[0, 1]$.
Following the setup in \cite{lu2016nonconvex,wang2015orthogonal},
we randomly sample $50\%$ of the pixels as observations.
For performance evaluation,
we use the root mean square error (RMSE) \cite{lu2016nonconvex} 
$\sqrt{\frac{1}{512^2} \NM{X - O}{F}^2} $,
where $O$ is the target image,
and $X$ is the recovered image.

\begin{figure}[ht]
\centering
\subfigure[\textit{Mountain}.]
{\includegraphics[width = 0.24\columnwidth]{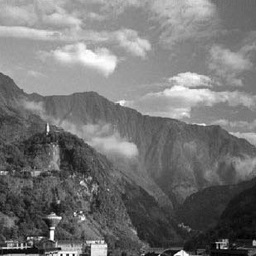}}
\quad
\subfigure[\textit{Sea}.]
{\includegraphics[width = 0.24\columnwidth]{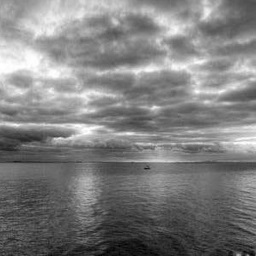}}
\quad
\subfigure[\textit{Windows}.]
{\includegraphics[width = 0.24\columnwidth]{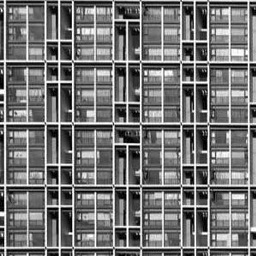}}

\vspace{-5px}
\caption{Clean images used for image completion.}
\label{fig:clnimg}
\end{figure}

We compare three types of algorithms:
nuclear norm regularization,
the factorization approach and nonconvex regularization.
For (convex) nuclear norm regularization,
we use 
\begin{enumerate}
\item AIS-Impute \cite{qyao2015ijcai}:
An inexact and accelerated proximal gradient algorithm; and
\item SketchCG \cite{yurtsever2017sketchy}: An efficient 
Frank-Wolfe 
variant 
with cheap iteration and low memory costs.
\end{enumerate}

For the factorization approach,
we use
\begin{enumerate}
\item ARSS-M3F \cite{yan2015scalable}, which is
	based on Riemannian manifold optimization; 
	\item LMaFit \cite{wen2012solving}, which
	factorizes 
	$X$ 
	as a product of two low-rank
	matrices, and then use alternating gradient descent for
	optimization; and
	\item ER1MP \cite{wang2015orthogonal}: 
	A greedy algorithm which increases the rank of the estimated matrix by one in each iteration.
	As suggested in \cite{wang2015orthogonal}, its economical version is used.
\end{enumerate}

Finally,
for the nonconvex regularization,
we use
\begin{enumerate}
	\item FaNCL \cite{qyao2015icdm}: 
	The state-of-the-art solver for matrix completion with nonconvex regularizers.
	It is based on 
	an efficient proximal gradient algorithm.
	Here, we use
	the log-sum-penalty regularizer (LSP) \cite{candes2008enhancing}, as
	it has the best reported performance in \cite{qyao2015icdm,lu2016nonconvex}; and
	
	\item nmAPG: proximal step is  computed in closed-form based on Proposition~\ref{pr:proxsol2},
	and the special structure on $Z$ (discussed at Section~\ref{sec:matcomp})
	is used to fast computation of its SVD.
\end{enumerate}

DCA \cite{matruncated2016} is not compared,
as nmAPG 
is much faster
(Section~\ref{sec:syn}).
The experiment is repeated 5 times.
On parameter tuning,
we set $\lambda = 0.1 \max_{i,j} \left|  \left[ S_{\Omega}(O) \right]_{ij} \right| $ in \eqref{eq:matcomp} for convex nuclear norm regularization and adaptive nonconvex regularization,
and then we set the rank as $200$ for factorization approaches.
These follow the suggestions in \cite{lu2016nonconvex,wang2015orthogonal}.

Table~\ref{tab:imgcomp} shows
the 
recovered RMSE
and 
running time of different algorithms.
As can be seen, factorization approaches (ARSS-M3F, LMaFit and ER1MP) are fast,
but their recovery performances are much inferior
to AIS-Impute which is based on the nuclear norm,
and FaNCL as well as nmAPG which are based on nonconvex regularizers.
SktechCG shares the same optimization problem as AIS-Impute,
but its recovery performance is not as good as AIS-Impute
as it is based on the Frank-Wolfe algorithm which suffers from slow convergence.
nmAPG achieves the lowest RMSE on all images.
The difference between recovered images and the clean ones 
are shown in Figure~\ref{fig:imgcomp}.
As can be seen, the image quality recovered from nmAPG is better than others.

\subsection{Image Denoising}
\label{sec:imgrec}


\begin{figure*}
\centering
\includegraphics[width = 0.32\textwidth]{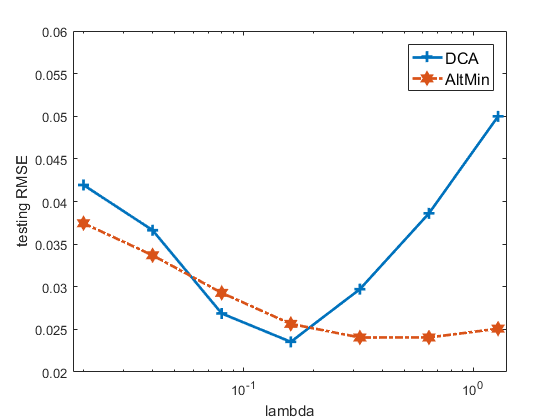}
\includegraphics[width = 0.32\textwidth]{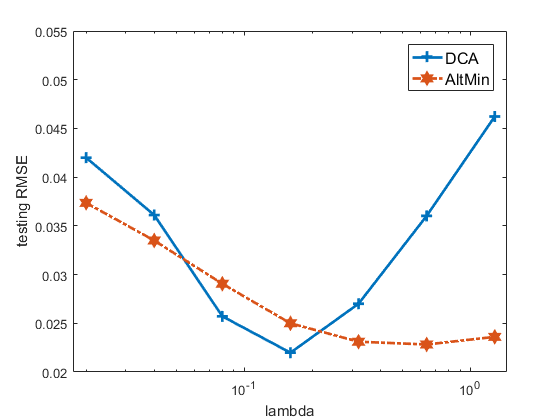}
\includegraphics[width = 0.32\textwidth]{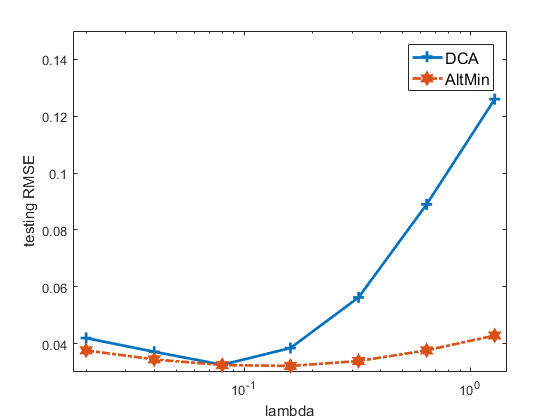}

\subfigure[\textit{Mountain}.]
{\includegraphics[width = 0.32\textwidth]{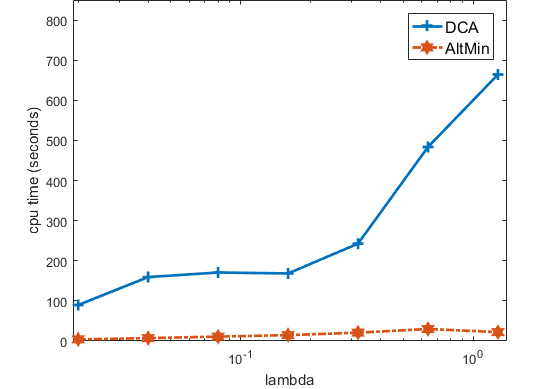}}
\subfigure[\textit{Sea}.]
{\includegraphics[width = 0.32\textwidth]{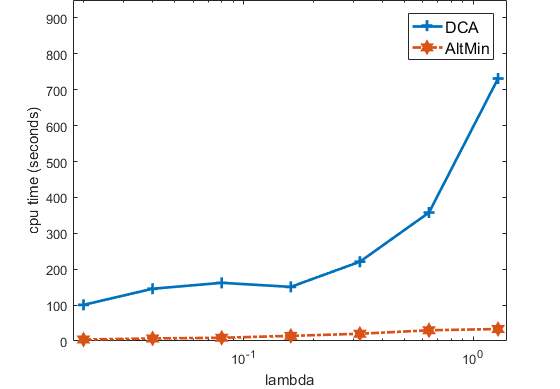}}
\subfigure[\textit{Windows}.]
{\includegraphics[width = 0.32\textwidth]{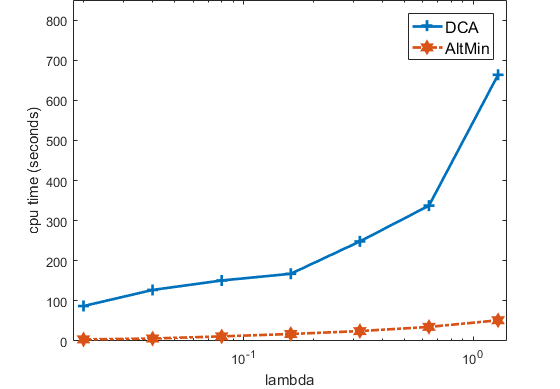}}

\caption{Testing RMSE v.s $\lambda$ (first row) and CPU time (in seconds) v.s $\lambda$ (second row) for AltMin and DCA algorithm.}
\label{fig:tv}
\end{figure*}

In this experiment, we use the total variation model for image denoising. 
The images in Figure~\ref{fig:clnimg} are used, 
and again pixels are normalized to $[0, 1]$.
Gaussian noise $\mathcal{N}(0, 0.05)$ is
added.
We compare the proposed solver AltMin (Algorithm~\ref{alg:altmin}) with DCA \cite{lou2015weighted}.
We do not compare with ADMM, 
as it does not have convergence guarantee as discussed in Section~\ref{sec:tvmdl}.
As this is a transductive problem with no validation set,
we vary $\lambda$ as $0.02 \times \{ 1, 2, \dots, 64  \}$.
The experiment is repeated $5$ times.
Figure~\ref{fig:tv} shows the RMSE and CPU time.
More detailed comparisons on recovered images are in
Figure~\ref{fig:tv:mountain}.
As can be seen, the proposed algorithm
yields comparable recovery performance as
DCA, but is about $20$ to $30$ times faster.

\begin{figure*}[ht]
\centering

\includegraphics[height = 0.12\textwidth]{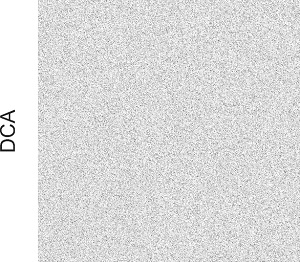}
\includegraphics[height = 0.12\textwidth]{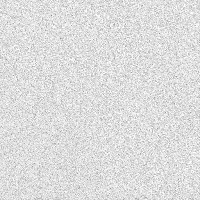}
\includegraphics[height = 0.12\textwidth]{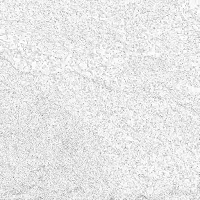}
\includegraphics[height = 0.12\textwidth]{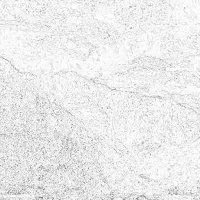}
\includegraphics[height = 0.12\textwidth]{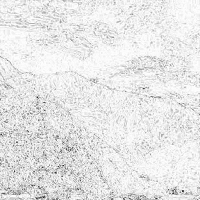}
\includegraphics[height = 0.12\textwidth]{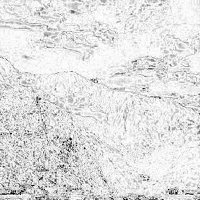}
\includegraphics[height = 0.12\textwidth]{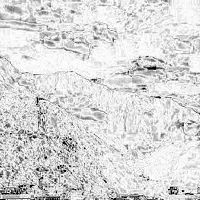}

\subfigure[$\lambda = 0.02$.]
{\includegraphics[height = 0.12\textwidth]{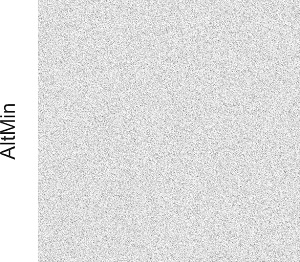}}
\subfigure[$\lambda = 0.04$.]
{{\includegraphics[height = 0.12\textwidth]{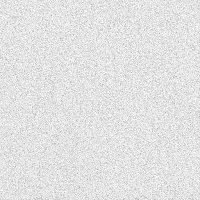}}}
\subfigure[$\lambda = 0.08$.]
{\includegraphics[height = 0.12\textwidth]{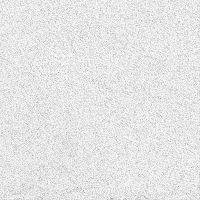}}
\subfigure[$\lambda = 0.16$.]
{\includegraphics[height = 0.12\textwidth]{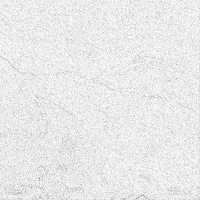}}
\subfigure[$\lambda = 0.32$.]
{\includegraphics[height = 0.12\textwidth]{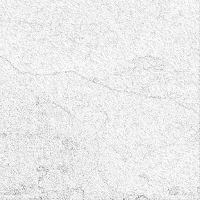}}
\subfigure[$\lambda = 0.64$.]
{\includegraphics[height = 0.12\textwidth]{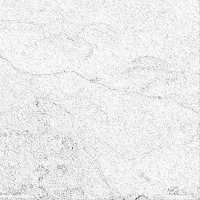}}
\subfigure[$\lambda = 1.28$.]
{\includegraphics[height = 0.12\textwidth]{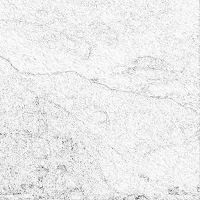}}

\caption{Difference of recovered images from DCA (first row) and AltMin (second row) algorithm to the clean image ``\textit{Mountain}''.
The darker pixels indicate larger difference.
Due to lack of space the figures for ``\textit{Sea}'' and ``\textit{Windows}'' are not shown.}
\label{fig:tv:mountain}
\vspace{-10px}
\end{figure*}

\section{Conclusion}

In this paper, 
we addressed the challenging optimization problem of nonconvex $\DifReg$ regularization.
We derived the closed-form solution for the associated proximal step.
This allows subsequent use of state-of-the-art proximal gradient algorithms.
We also extend the results for low-rank matrix learning and total variation model.
Experimental results show that the proximal step can be computed very efficiently.
Superiority of the $\DifReg$-regularizer over other nonconvex regularizers is 
also demonstrated on real data sets.

As for the future works, 
it is interesting to consider stochastic optimization algorithms,
such as stochastic variance reduction gradient descent (SVRG) algorithm \cite{johnson2013accelerating}, 
with the $\DifReg$ regularization.
Although we have solved the problem with proximal step,
the convergence of SVRG algorithm is still not clear for such a nonconvex regularizer.

{
\bibliographystyle{IEEEtran}
\bibliography{bib}
}

\ifCLASSOPTIONcaptionsoff
  \newpage
\fi

\vspace{-40px}

\begin{IEEEbiography}[{\includegraphics[width = 1\textwidth]{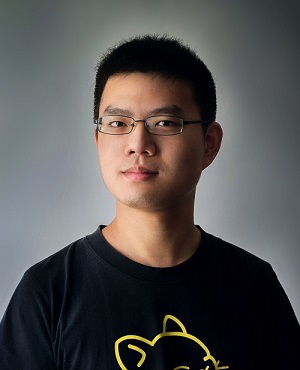}}]{Quanming Yao}
	received the bachelor’s degree in Electronic and Information Engineering from the  Huazhong University of Science and Technology (HUST) in
	2013. Currently, he is working toward the PhD
	degree in the Department of Computer Science and Engineering
	in the Hong Kong
	University of Science and Technology. His
	research interests focus on machine learning,
	data mining, application problems on computer
	vision and other problems in artificial intelligence.
	He was awarded as Qiming star of HUST in 2012,
	and received the Google PhD fellowship (machine learning) in 2016.
\end{IEEEbiography}

\vspace{-40px}

\begin{IEEEbiography}[{\includegraphics[width = 1\textwidth]{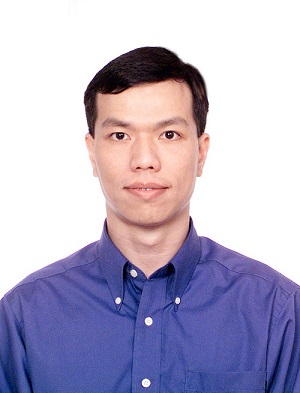}}]{James T. Kwok}
	received the PhD degree in
	computer science from the Hong Kong University
	of Science and Technology in 1996. He was
	with the Department of Computer Science,
	Hong Kong Baptist University, Hong Kong, as
	an assistant professor. He is currently a professor
	in the Department of Computer Science and
	Engineering, Hong Kong University of Science
	and Technology. His research interests include
	kernel methods, machine learning, example recognition,
	and artificial neural networks. He
	received the IEEE Outstanding 2004 Paper Award, and the Second
	Class Award in Natural Sciences by the Ministry of Education, People’s
	Republic of China, in 2008. He has been a program cochair for a number
	of international conferences, and served as an associate editor for
	the IEEE Transactions on Neural Networks and Learning Systems
	from 2006-2012. Currently, he is an associate editor for the Neurocomputing
	journal.
\end{IEEEbiography}

\vspace{-40px}

\begin{IEEEbiography}[{\includegraphics[width = 1\textwidth]{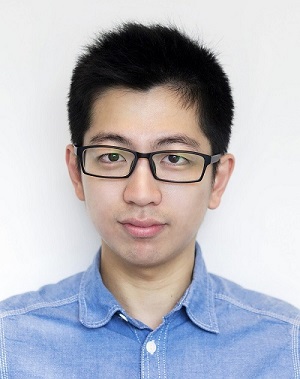}}]{Xiawei Guo}
	received the bachelors degree in Electronic Information Science and Technology from Nanjing University in 2013. Currently, he is working towards the MPhil degree in computer science at the Hong Kong Univiersity of Science and Technology. His research interests focus on machine learning, data mining and other problems in artificial intelligence.
\end{IEEEbiography}






\appendices

\end{document}